\documentclass[final,12pt]{arxiv} 

\usepackage[english]{babel}
\usepackage{sectsty}
\usepackage{url}
\usepackage{caption}
\usepackage{Definitions}
\usepackage{mathtools}
\usepackage{algorithmic}
\usepackage{algorithm}
\usepackage{xifthen}
\usepackage{xparse}
\usepackage{dsfont}

\newcommand{\lr}[1]{\left (#1 \right)}
\newcommand{\lrc}[1]{\left \{ #1 \right \}}
\newcommand{\lrs}[1]{\left [ #1 \right ]}
\newcommand{\innerprod}[2]{\langle #1, #2\rangle}

\NewDocumentCommand{\E}{o}{\mathbb E\IfValueT{#1}{\lrs{#1}}}
\NewDocumentCommand{\1}{o}{\mathds 1{\IfValueT{#1}{\lr{#1}}}}
\let\P\undefined
\NewDocumentCommand{\P}{o}{\mathbb P{\IfValueT{#1}{\lr{#1}}}}

\newcommand{\Reg}{\overline{Reg}_T}

\title[Improved Analysis of the Tsallis-INF Algorithm]{Improved Analysis of the Tsallis-INF Algorithm in Stochastically Constrained Adversarial Bandits and Stochastic Bandits with Adversarial Corruptions}
\usepackage{times}

\coltauthor{\Name{Saeed Masoudian} \Email{saeed.masoudian@di.ku.dk}\\
 \Name{Yevgeny Seldin} \Email{seldin@di.ku.dk}\\
 \addr University of Copenhagen, Denmark}

\begin{document}

\maketitle

\begin{abstract}%
We derive improved regret bounds for the Tsallis-INF algorithm of Zimmert and Seldin (2021). We show that in adversarial regimes with a $(\Delta,C,T)$ self-bounding constraint the algorithm achieves 
$\mathcal{O}\left(\left(\sum_{i\neq i^*} \frac{1}{\Delta_i}\right)\log_+\left(\frac{(K-1)T}{\left(\sum_{i\neq i^*} \frac{1}{\Delta_i}\right)^2}\right)+\sqrt{C\left(\sum_{i\neq i^*}\frac{1}{\Delta_i}\right)\log_+\left(\frac{(K-1)T}{C\sum_{i\neq i^*}\frac{1}{\Delta_i}}\right)}\right)$ regret bound, where  $T$ is the time horizon, $K$ is the number of arms, $\Delta_i$ are the suboptimality gaps, $i^*$ is the best arm, $C$ is the corruption magnitude, and $\log_+(x) = \max\left(1,\log x\right)$. The regime includes stochastic bandits, stochastically constrained adversarial bandits, and stochastic bandits with adversarial corruptions as special cases. 
Additionally, we provide a general analysis, which allows to achieve the same kind of improvement for generalizations of Tsallis-INF to other settings beyond multiarmed bandits.
\end{abstract}



\section{Introduction}

Most of the literature on multiarmed bandits is focused either on the stochastic setting \citep{thompson1933,robbins1952,lai1985,auer2002a} or on the adversarial one \citep{auer2002b}. However, in recent years there has been an increasing interest in algorithms that perform well in both regimes with no prior knowledge of the regime \citep{bubeck2012b,seldin2014,auer2016,seldin2017,wei2018}, as well as algorithms that perform well in intermediate regimes between stochastic and adversarial \citep{seldin2014,lykouris2018,wei2018,gupta2019}. The quest for best-of-both-worlds algorithm culminated with the work of \citet{zimmert2019}, who proposed the Tsallis-INF algorithm and showed that its regret bound in both stochastic and adversarial environments matches the corresponding lower bounds within constants with no need of prior knowledge of the regime. \citet{zimmert2020} further improved the analysis and introduced an \textit{adversarial regime with a self-bounding constraint}, which is an intermediate regime between stochastic and adversarial environments, including \textit{stochastically constrained adversaries} \citep{wei2018} and \textit{stochastic bandits with adversarial corruptions} \citep{lykouris2018} as special cases. They have shown that the Tsallis-INF algorithm achieves the best known regret rate in this regime and its special cases.

The Tsallis-INF algorithm is based on regularization by Tsallis entropy with power $\frac{1}{2}$, which was also used in the earlier works by \citet{audibert2009,audibert2010} and \citet{abernethy2015} for minimax optimal regret rates in the adversarial regime. The key novelty of the work of \citet{zimmert2019, zimmert2020} is an analysis of the algorithm in the stochastic setting based on a self-bounding property of the regret. The idea has been subsequently extended to derive best-of-both-worlds algorithms for combinatorial semi-bandits \citep{zimmert-luo-wei-icml-2019}, decoupled exploration and exploitation \citep{rouyer-seldin-colt-2020}, bandits with switching costs \citep{RSCB21}, and ergodic MDPs \citep{jin-luo-2020}.

We present a refined analysis based on the self-bounding property, which improves the regret bound in the adversarial regime with a self-bounding constraint and its special cases: stochastic bandits, stochastically constrained adversarial bandits, and stochastic bandits with adversarial corruption. The adversarial regime with a self-bounding constraint is defined in the following way. Let $\ell_1,\ell_2,\dots$ be a sequence of loss vectors with $\ell_t \in [0,1]^K$, let $I_t$ be the action picked by the algorithm at round $t$, and let $\overline{Reg}_T = \E[\sum_{t=1}^T \ell_{t,I_t}] - \min_i \E[\sum_{t=1}^T \ell_{t,i}]$ be the pseudo-regret. For a triplet $\lr{\Delta, C, T}$ with $\Delta \in [0,1]^K$ and $C\geq 0$, \citet{zimmert2020} define an \textit{adversarial regime with a $\lr{\Delta,C,T}$ self-bounding constraint} as an adversarial regime, where the adversary picks losses, such that the pseudo-regret of any algorithm at time $T$ satisfies
\[
\overline{Reg}_T \geq \sum_{t=1}^T \sum_i \Delta_i \P[I_t = i] - C.
\]
(The above condition is only assumed to be satisfied at time $T$, but there is no requirement that it is satisfied at time $t<T$.) A special case of this regime is the stochastically constrained adversarial regime, where $\overline{Reg}_T = \sum_{t=1}^T \sum_i \Delta_i \P[I_t = i]$ with $\Delta$ being the vector of suboptimality gaps. In particular, the stochastic regime is a special case of the stochastically constrained adversarial regime. (In the stochastic regime the expected loss of each arm is fixed over time. Stochastically constrained adversarial regime relaxes this requirement by only assuming that the expected gaps between the losses of pairs of arms are fixed, but the expected losses are allowed to fluctuate over time.) Another special case of an adversarial regime with a self-bounding constraint are stochastic bandits with adversarial corruptions. For two sequences of losses $\overline{\mathcal{L}}_T = (\bar \ell_1, \dots, \bar \ell_T)$ and $\mathcal{L}_T = (\ell_1, \dots, \ell_T)$ the amount of corruption is measured by $\sum_{t=1}^T \|\bar \ell_t - \ell_t\|_\infty$. In stochastic bandits with adversarial corruptions the adversary takes a stochastic sequence of losses and injects corruption with corruption magnitude bounded by $C$. \citet{zimmert2020} show that a stochastic, as well as a stochastically constrained adversarial regime with a vector of suboptimality gaps $\Delta$ and injected corruption of magnitude bounded by $C$, satisfy $(\Delta, 2C, T)$ self-bounding constraint. As $C$ grows from zero to $T$, the stochastic regime with adversarial corruptions interpolates between stochastic and adversarial bandits.

\citet{lykouris2018} were the first to introduce and study stochastic bandits with adversarial corruptions and their algorithm achieved $\Ocal\left( \sum_{i: \Delta_i > 0} \frac{KC+\log(T)}{\Delta_i}\log(T)\right)$ regret bound. \citet{gupta2019} improved it to $\Ocal\left(KC + \sum_{i: \Delta_i > 0} \frac{1}{\Delta_i}\log^2(KT)\right)$. \citet{zimmert2020} have shown that their best-of-both-worlds Tsallis-INF algorithm achieves $\Ocal\lr{\lr{\sum_{i\neq i^*} \frac{\log T}{\Delta_i}}+\sqrt{C\sum_{i\neq i^*}\frac{\log T}{\Delta_i}}}$ regret bound in the more general adversarial regime with $(\Delta, C, T)$ self-bounding constraint under the assumption that $\Delta$ has a unique zero entry (the assumption corresponds to uniqueness of the best arm \textit{before} corruption). Neither of the algorithms requires prior knowledge of $C$.

Our contributions are summarized in the enumerated list below. The improvements relative to the work by \citet{zimmert2020} are further highlighted in Table~\ref{tab:results}.
\begin{table}[t]
    \centering
    \begin{tabular}{c|c|c}
    Setting & \cite{zimmert2020} & Our paper\\
     \hline
     Small $C$ & $\Ocal\lr{\sum_{i\neq i^*} \frac{1}{\Delta_i}\log T}$& $\Ocal\lr{\sum_{i\neq i^*} \frac{1}{\Delta_i}\log_+ \lr{T{\color{red}\frac{K-1}{(\sum_{i\neq i^*} 1/\Delta_i)^2}}}}$\\
     Large $C$ & $\Ocal\lr{\sqrt{C\sum_{i\neq i^*} \frac{1}{\Delta_i} \log T}}$ & $\Ocal\lr{\sqrt{C\sum_{i\neq i^*} \frac{1}{\Delta_i} \log_+ \lr{T{\color{red}\frac{K-1}{{\color{magenta}C} (\sum_{i\neq i^*} 1/\Delta_i)}}}}}$
    \end{tabular}
    \caption{Comparison of the leading terms in the regret bounds of \cite{zimmert2020} and our paper, differences are highlighted in color. We define $\log_+(x) = \max\lr{1,\log x}$. The "Small $C$" row compares the regret bounds in adversarial regimes with $(\Delta,C,T)$ self-bounding constraints with $C \leq \sum_{i\neq i*} \frac{1}{\Delta_i}\lr{\lr{\log\frac{T(K-1)}{(\sum_{i\neq i*} \frac{1}{\Delta_i})^2}}+1}$. Here, $C$ is a subdominant term and does not show up in the big-$\Ocal$ notation. The "Large $C$" row compares the regret bounds in adversarial regimes with $(\Delta,C,T)$ self-bounding constraints with $C \geq \sum_{i\neq i*} \frac{1}{\Delta_i}\lr{\lr{\log\frac{T(K-1)}{(\sum_{i\neq i*} \frac{1}{\Delta_i})^2}}+1}$. The regret bounds in the adversarial regime are identical, and hence omitted.}
    \label{tab:results}
\end{table}
\begin{enumerate}
    \item We present a refined analysis of the regret of Tsallis-INF in adversarial regimes with a $(\Delta,C,T)$ self-bounding constraint, achieving 
    \[\Ocal\lr{\lr{\sum_{i\neq i^*} \frac{1}{\Delta_i}}\log_+\lr{\frac{(K-1)T}{\lr{\sum_{i\neq i^*} 
    \frac{1}{\Delta_i}}^2}}+\sqrt{C\lr{\sum_{i\neq i^*}\frac{1}{\Delta_i}}\log_+\lr{\frac{(K-1)T}{C\sum_{i\neq i^*}\frac{1}{\Delta_i}}}}}
    \]
    regret bound, where $\log_+(x) = \max\lr{1,\log x}$. 
    \item In the stochastically constrained adversarial regime it improves the dominating term of the regret bound from $\Ocal\lr{\lr{\sum_{i\neq i^*}\frac{1}{\Delta_i}}\log T}$ to $\Ocal\lr{\lr{\sum_{i\neq i^*} \frac{1}{\Delta_i}}\log_+\lr{\frac{(K-1)T}{\lr{\sum_{i\neq i^*} 
    \frac{1}{\Delta_i}}^2}}}$ relative to the work of \citet{zimmert2020}, see Table~\ref{tab:results}. A similar kind of improvement has been studied for UCB-type algorithms for stochastic bandits by \citet{AO10} and \citet{Lat18}.
    \item In the stochastic regime with adversarial corruptions the result yields an improvement by a multiplicative factor of $\Ocal\lr{\sqrt{\log T/\log \lr{T/C}}}$ relative to the work of \citet{zimmert2020}, see Table~\ref{tab:results} for a more refined statement. In particular, for $C = \Theta\lr{\frac{TK}{(\log T) \sum_{i\neq i^*} \frac{1}{\Delta_i}}}$ it achieves an improvement by a multiplicative factor of $\sqrt{\frac{\log T}{\log \log T}}$.
    \item While the analysis of \citet{zimmert2020} used two different optimization problems to analyze the regret of Tsallis-INF in adversarial environments and in adversarial environments with a self-bounding constraint, we obtain both bounds from the same optimization problem. This provides continuity in the analysis in the sense that the $\Ocal\lr{\sqrt{KT}}$ adversarial regret bound is obtained as a natural limit case of the adversarial bound with a self-bounding constraint as $C$ grows beyond $\Ocal\lr{\frac{KT}{\sum_{i\neq i^*}\frac{1}{\Delta_i}}}$. It also provides a better understanding of the self-bounding analysis technique.
    \item We also provide a more general result, showing that any algorithm with adversarial pseudo-regret bound satisfying $\Reg \leq B\sum_{t=1}^T\sum_{i\neq i^*} \sqrt{\frac{\E[w_{t,i}]}{t}}$, where $w_{t,i}$ are the probabilities of playing action $i$ at round $t$ and $B$ is a constant, 
    achieves\[\Ocal\lr{B^2\lr{\sum_{i\neq i^*} \frac{1}{\Delta_i}}\log_+\lr{\frac{(K-1)T}{\lr{\sum_{i\neq i^*} 
    \frac{1}{\Delta_i}}^2}} + B\sqrt{C\lr{\sum_{i\neq i^*} \frac{1}{\Delta_i}}\log_+\lr{\frac{KT}{C\sum_{i\neq i^*}\frac{1}{\Delta_i}}}}}\] regret in the adversarial regime with $(\Delta, C, T)$ self-bounding constraint.
    The result can be directly applied to achieve improved regret bounds for extensions of the Tsallis-INF algorithm, for example, the extension to episodic MDPs \citep{jin-luo-2020}.
\end{enumerate}
\section{Problem Setting}

 We study multi-armed bandit problem in which at time $t = 1,2,\ldots$ the learner chooses an arm $I_t$ among a set of $K$ arms $\{1,\ldots, K\}$. At the same time the environment selects a loss vector $\ell_t \in [0,1]^K$ and the learner only observes and suffers the loss $\ell_{t,I_t}$. The performance of the learner is evaluated using \regret, which is defined as
\[
\overline{Reg}_T = \EE\left[\sum_{t=t}^{T} \ell_{t,I_t}\right] - \min_{i \in [K]} \EE \left[\sum_{t=t}^{T}  \ell_{t,i} \right] = \EE\left[\sum_{t=t}^{T} \left(\ell_{t,I_t} - \ell_{t,i_T^*}\right) \right],
\]
where $i_T^* \in \argmin_{i \in [K]} {\EE\left[\sum_{t=t}^{T}  \ell_{t,i}\right] } $ is a best arm in hindsight in expectation over the loss generation model and, in case of an adaptive adversary, the randomness of the learner. 

Like \citet{zimmert2020} we consider (\emph{adaptive}) \emph{adversarial regimes} and \emph{adversarial regimes with a $(\Delta,C,T)$ self-bounding constraint}. In the former the losses at round $t$ are generated arbitrarily, potentially depending on the preceding actions of the learner, $I_1\dots,I_{t-1}$. In the latter the adversary selects losses, such that for some $\Delta \in [0,1]^K$ and $C\geq 0$ the \regret of any algorithm at time $T$ satisfies
\begin{equation}\label{eq:self-bounding}
    \overline{Reg}_T \geq  \lr{\sum_{t = 1}^{T} \sum_{i =1}^K \PP(I_t = i)  \Delta_i} - C.
\end{equation}
The condition is only assumed to be satisfied at time $T$, but not necessarily at $t<T$. As we have already mentioned in the introduction, \textit{stochastic} regime, \textit{stochastically constrained adversarial} regime, and \textit{stochastic bandits with adversarial corruptions} are all  special cases of the adversarial regime with $(\Delta,C,T)$ self-bounding constraint.

\paragraph{Additional Notation:}
We use $\Delta^{n}$ to denote the probability simplex over $n+1$ points. The characteristic function of a closed convex set $\Acal$ is denoted by $\Ical_{\Acal}(x)$ and satisfies $\Ical_{\Acal}(x) = 0$ for $x \in \Acal$ and $\Ical_{\Acal}(x) = \infty$ otherwise. We denote the indicator function of an event $\Ecal$ by $\1[\Ecal]$ and use $\1_t(i)$ as a shorthand for $\1[I_t = i]$. The probability distribution over arms that is played by the learner at round $t$ is denoted by $w_t \in \Delta^{K-1}$. The convex conjugate of a function $f: \RR^n \rightarrow \RR$ is defined by $f^*(y) = \sup_{x \in \RR^n} \{\langle x,y \rangle - f(x) \}$.

\section{Background: the Tsallis-INF algorithm}

In this section we provide a brief background on the Tsallis-INF algorithm of \citet{zimmert2020}. The algorithm is based on Follow The Regularized Leader (FTRL) framework with Tsallis entropy regularization \citep{tsallis1988}. The best-of-both-worlds version of Tsallis-INF uses Tsallis entropy regularizer with power $\frac{1}{2}$, defined by
\[
\Psi(w) = 4\sum_{i = 1}^K \lr{\sqrt{w_i} - \frac{1}{2} w_i}.
\]
The regularization term at round $t$ is given by 
\[
\Psi_t(w) = \eta_t^{-1}\Psi(w),
\]
where $\eta_t$ is the learning rate. The update rule for the distribution over actions is defined by
\[
w_{t+1} = \nabla (\Psi_t + \Ical_{\Delta^{K-1}})^*(-\sum_{\tau = 1}^t \hat{\ell}_\tau ) = \arg\max_{w\in\Delta^{K-1}}\lr{\left\langle-\sum_{\tau = 1}^t \hat{\ell}_\tau,w\right\rangle - \Psi_t(w)},
\]
where $\hat{\ell}_\tau$ is an estimate of the loss vector $\ell_\tau$. It is possible to use the standard importance-weighed loss estimate $\hat \ell_{t,i} = \frac{\ell_{t,i}\1[\Ical_t=i]}{w_{t,i}}$, but \citet{zimmert2020} have shown that reduced-variance loss estimates defined by
\begin{equation}\label{eq:rv}
\hat{\ell}_{t,i} = \frac{\1_t(i)(\ell_{t,i}-\BB_t(i))}{w_{t,i}} + \BB_{t}(i),
\end{equation}
where $\BB_t(i) = \frac{1}{2} \1[w_{t,i} \geq \eta_t^2]$, lead to better constants. The complete algorithm is provided in Algorithm \ref{alg:tsallis} box. The regret bound derived by \citet{zimmert2020} is provided in Theorem~\ref{thorem:1}.

\begin{algorithm}[H]
\caption {Tsallis-INF}
\begin{algorithmic}[1]
    \STATE {\textbf{Input:} $(\Psi_t)_{t=1,2,\ldots}$ }
    \STATE {\textbf{Initialize:} Set $\hat{L}_0 = \mathbf{0}_K$ (where $\mathbf{0}_K$ is a zero vector in $\RR^K$) }
    \FOR{t = 1, \ldots}
    \STATE {choose $w_t = \nabla (\Psi_t + \Ical_{\Delta^{K-1}})^*(-\hat{L}_{t-1})$}
    \STATE {sample $I_t \sim w_t$}
    \STATE {observe $\ell_{t, I_t}$}
    \STATE {construct a loss estimator $\hat{\ell}_t$ using \eqref{eq:rv}}
    \STATE {update $\hat{L}_t =\hat{L}_{t-1} + \hat{\ell}_t $}
    \ENDFOR
\end{algorithmic}
\label{alg:tsallis}
\end{algorithm}

\begin{theorem}[\citealp{zimmert2020}]\label{thorem:1}
The \regret of Tsallis-INF with $\eta_t = \frac{4}{\sqrt{t}}$ and reduced variance loss estimators defined in equation \eqref{eq:rv}, in any adversarial bandit problem satisfies
\[
\Reg \leq 2\sqrt{KT}+10K\log(T)+16.
\]
Furthermore, if there exists a vector $\Delta \in [0,1]^K$ with a unique zero entry $i^*$ (i.e., $\Delta_{i^*} = 0$ and $\Delta_i > 0$ for all $i\neq i^*$) and a constant $C$, such that the pseudo-regret at time $T$ satisfies the $(\Delta, C, T)$ self-bounding constraint (equation \eqref{eq:self-bounding}), then the 
\regret additionally satisfies:
\begin{equation}
\label{eq:ZS21sto}
\overline{Reg}_T \leq \left(\sum_{i\neq i^*}\frac{\log(T)+3}{\Delta_i} \right) + 28K\log(T) + \frac{1}{\Delta_{min}} + \frac{3}{2}\sqrt{K} + 32 + C, 
\end{equation}
where $\Delta_{min} = \min_{i\neq i^*}\{\Delta_i\}$. Moreover, if $C \geq \left(\sum_{i\neq i^*}\frac{\log(T)+3}{\Delta_i} \right) + \frac{1}{\Delta_{min}}$, then the \regret also satisfies:
\begin{equation}
\label{eq:ZS21stoC}
\overline{Reg}_T  \leq 2\sqrt{\left(\sum_{i\neq i^*}\frac{\log(T)+3}{\Delta_i} + \frac{1}{\Delta_{min}} \right)C} + 28K\log(T) + \frac{3}{2}\sqrt{K} + 32.
\end{equation}
\end{theorem}

\begin{remark}
While Theorem~\ref{thorem:1} requires uniqueness of the best arm for improved regret rates in the adversarial regime with a $(\Delta, C, T)$ self-bounding constraint, \citet{zimmert2020} have shown experimentally that in the stochastic regime the presence of multiple best arms has no negative effect on the \regret of the algorithm. They conjecture that the requirement is an artifact of the analysis.
\end{remark}

\section{Main Results}\label{sec:method}

In this section we provide our two main results. First, in Theorem~\ref{theorem:2} we provide a refined analysis of Tsallis-INF, which improves the \regret bounds in the adversarial regime with a $(\Delta,C,T)$ self-bounding constraint. Then, in Theorem~\ref{theorem:3} we provide a more general result, which allows to improve \regret bounds in adversarial regimes with $(\Delta,C,T)$ self-bounding constraints for extensions of Tsallis-INF to other problems. An advantage of both results is that the bounds for adversarial regimes and adversarial regimes with a self-bounding constraint are achieved from a single optimization problem, rather than from two different optimization problems, as in prior work. As a result, the regret bounds for the adversarial regime are achieved as a limit case of the regret bounds for adversarial regimes with a self-bounding constraint for large $C$. 

\subsection{Improved analysis of the Tsallis-INF algorithm}

We start with an improved regret bound for Tsallis-INF.
\begin{theorem}\label{theorem:2}
The \regret of Tsallis-INF with $\eta_t = \frac{4}{\sqrt{t}}$ and reduced variance loss estimators defined in equation \eqref{eq:rv}, in any adversarial bandit problem satisfies
\begin{equation}
\label{eq:adv}
\Reg \leq 2\sqrt{(K-1)T}+\frac{1}{2}\sqrt{T}+14K\log(T) + \frac{3}{4}\sqrt{K} + 15.
\end{equation}
Furthermore, if there exists a vector $\Delta \in [0,1]^K$ with a unique zero entry $i^*$ (i.e., $\Delta_{i^*} = 0$ and $\Delta_i > 0$ for all $i\neq i^*$) and a constant $C\geq 0$, such that the pseudo-regret at time $T$ satisfies the $(\Delta, C, T)$ self-bounding constraint (equation \eqref{eq:self-bounding}), then the \regret additionally satisfies:
\begin{equation}
\label{eq:sto}
    \overline{Reg}_T \leq
    \sum_{i\neq i*} \frac{1}{\Delta_i}\lr{\lr{\log \frac{T(K-1)}{\lr{\sum_{i\neq i^*} \frac{1}{\Delta_i}}^2}} + 6} + 28K \log(T) + \frac{3}{2}\sqrt{K} + 30 + C.
\end{equation}
Moreover, for $\sum_{i\neq i*} \frac{1}{\Delta_i}\lr{\lr{\log\frac{T(K-1)}{(\sum_{i\neq i*} \frac{1}{\Delta_i})^2}}+1} \leq C \leq \frac{T(K-1)}{\sum_{i\neq i*} \frac{1}{\Delta_i}}$ the regret also satisfies:
\begin{equation}
\label{eq:stoC}
    \Reg \leq \sqrt{C \sum_{i\neq i^*} \frac{1}{\Delta_i}}\lr{\sqrt{\log\frac{T(K-1)}{C\sum_{i\neq i^*} \frac{1}{\Delta_i}}} + 5} + Q, 
\end{equation}
    where $Q = \sum_{i\neq i^*} \frac{1}{\Delta_i}\lr{\log\frac{T(K-1)}{C\sum_{i\neq i^*} \frac{1}{\Delta_i}} + \sqrt{2\log\frac{T(K-1)}{C\sum_{i\neq i^*} \frac{1}{\Delta_i}}}+ 2} + \frac{3\sqrt{K}}{2} + 28K \log(T) + 30
    $ is a subdominant term.
\end{theorem}
A proof of the theorem is provided in Appendix~\ref{section:appendix}. Theorem~\ref{theorem:2} improves on Theorem~\ref{thorem:1} in two ways. The bound in equation \eqref{eq:sto} improves the leading term of the regret bound under self-bounding constraint relative to equation \eqref{eq:ZS21sto} from $\sum_{i\neq i^*}\frac{1}{\Delta_i}\log T$ to $\sum_{i\neq i*} \frac{1}{\Delta_i}\lr{\log \frac{T(K-1)}{\lr{\sum_{i\neq i^*} \frac{1}{\Delta_i}}^2}}$. Related refinements of regret bounds for UCB strategies for ordinary stochastic bandits have been studied by \citet{AO10} and \citet{Lat18}. More importantly, 
for large amount of corruption 
$C \in \lrs{\sum_{i\neq i*} \frac{1}{\Delta_i}\lr{\log\lr{\frac{T(K-1)}{(\sum_{i\neq i*} \frac{1}{\Delta_i})^2}}+1},  \frac{T(K-1)}{\sum_{i\neq i*} \frac{1}{\Delta_i}}}$ the regret bound in equation \eqref{eq:stoC} is of order $\Ocal\lr{\sqrt{C\lr{\sum_{i\neq i^*}\frac{1}{\Delta_i}}\log_+\lr{\frac{KT}{C\sum_{i\neq i^*}\frac{1}{\Delta_i}}}}}$, whereas the regret bound in equation \eqref{eq:ZS21stoC} is of order  $\Ocal\lr{\sqrt{C\sum_{i\neq i^*}\frac{\log T}{\Delta_i}}}$. For $C = \Theta\lr{\frac{TK}{(\log T) \sum_{i\neq i^*} \frac{1}{\Delta_i}}}$ Theorem~\ref{theorem:2} improves the \regret bound by a multiplicative factor of $\sqrt{\frac{\log T}{\log \log T}}$. Another observation is that Theorem~\ref{theorem:2} successfully exploits the self-bounding property even when the amount of corruption is almost linear in $T$.

\subsection{A general analysis based on the self-bounding property}\label{sec:sqrtcondition}

Now we provide a general result, which can be used to analyze extensions of Tsallis-INF to other problem settings.
\begin{theorem}\label{theorem:3}
For any algorithm for an arbitrary problem domain with $K$ possible actions that satisfies
\begin{equation}\label{eq:sqrt}
    \overline{Reg}_T  \leq B \sum_{t=1}^T \sum_{i \neq i^* } \sqrt{\frac{\EE[w_{t,i}]}{t}} + D,
\end{equation}
where $B,D \geq 0$ are some constants, the \regret of the algorithm in any adversarial environment satisfies
\begin{equation}
\label{eq:Badv}
\Reg \leq 2B\sqrt{(K-1)T} + D.
\end{equation}
Furthermore, if there exists a vector $\Delta \in [0,1]^K$ with a unique zero entry $i^*$ (i.e., $\Delta_{i^*} = 0$ and $\Delta_i > 0$ for all $i\neq i^*$) and a constant $C\geq 0$, such that the pseudo-regret at time $T$ satisfies the $(\Delta, C, T)$ self-bounding constraint (equation \eqref{eq:self-bounding}), then the \regret additionally satisfies:
\begin{equation}
\label{eq:Bsto}
    \overline{Reg}_T \leq B^2 \sum_{i\neq i*} \frac{1}{\Delta_i} \lr{\lr{\log\frac{T(K-1)}{\lr{\sum_{i\neq i^*}\frac{1}{\Delta_i}}^2}} + 3 - 2\log B} + C + 2D.
\end{equation}
Moreover, for $B^2 \sum_{i\neq i*} \frac{1}{\Delta_i}\lr{\lr{\log\frac{T(K-1)}{B^2(\sum_{i\neq i*} \frac{1}{\Delta_i})^2}}+1} \leq C \leq \frac{T(K-1)}{\sum_{i\neq i*} \frac{1}{\Delta_i}}$ the \regret also satisfies:
\begin{equation}
\label{eq:BstoC}
    \overline{Reg}_T \leq B\sqrt{C \sum_{i\neq i^*}\frac{1}{\Delta_i}} \lr{\sqrt{\log\frac{T(K-1)}{C\sum_{i\neq i^*}\frac{1}{\Delta_i}}} + 2} + M, 
\end{equation}
where $M = 
B^2 \sum_{i\neq i^*}\frac{1}{\Delta_i}\lr{\log\frac{T(K-1)}{C\sum_{i\neq i^*}\frac{1}{\Delta_i}}+\sqrt{2\log\frac{T(K-1)}{C\sum_{i\neq i^*}\frac{1}{\Delta_i}}}+2} + 2D$ is a subdominant term.
\end{theorem}
A proof is provided in Section~\ref{sec:proofs}. The Tsallis-INF algorithm satisfies the condition in equation \eqref{eq:sqrt} with $B=\frac{5}{4}$ (see equation \eqref{eq:zimmertlemma} in Section~\ref{sec:proofs}, which follows from intermediate results by \citet{zimmert2020}). Although the specialized analysis of Tsallis-INF in Theorem~\ref{theorem:2} is a bit tighter than the general result in Theorem~\ref{theorem:3}, the latter can be applied to extensions of Tsallis-INF. One such example is the best-of-both-worlds algorithm of \citet{jin-luo-2020} for episodic MDPs. \citet[Theorem 4]{jin-luo-2020} show that their algorithm satisfies the condition in \eqref{eq:sqrt} and use this result to achieve $\Ocal\lr{\lr{\log T} + \sqrt{C \log(T)}}$ \regret bound in the stochastic case with adversarial corruptions \citep[Corollary 3]{jin-luo-2020}. Application of our Theorem~\ref{theorem:3} improves the \regret bound to $\Ocal\lr{\lr{\log T} + \sqrt{C \log(T/C)}}$. In particular, for $C = \Theta(\frac{T}{\log T})$ the bound gets tighter by a multiplicative factor of $\frac{\log T}{\log \log T}$.
\section{Proofs}\label{sec:proofs}

In this section we provide a proof of Theorem~\ref{theorem:3}. The proof of Theorem~\ref{theorem:2} is analogous, but more technical due to fine-tuning of the constants and is deferred to Appendix \ref{section:appendix}. Before showing the proof we revisit the key steps in the analysis of Tsallis-INF by \citet{zimmert2020}, which show that the \regret of Tsallis-INF satisfies the condition in equation \eqref{eq:sqrt} of Theorem~\ref{theorem:3}.

Standard FTRL analysis  \citep{lattimore2020} uses a potential function $\Phi_t(-L) = \max_{ w \in \Delta^{K-1}} \{ \innerprod{w}{-L} - \Psi_t(w) \}$  for breaking the \regret into \emph{penalty} and \emph{stability} terms, $\Reg = stability + penalty$, where
\begin{align*}
    stability &= \EE\lrs{\sum_{t=1}^T \ell_{t,I_t} + \Phi_t(-\hat{L}_{t}) - \Phi_t(-\hat{L}_{t-1})},\\
    penalty &= \EE\lrs{\sum_{t=1}^T -\Phi_t(-\hat{L}_{t}) + \Phi_t(-\hat{L}_{t-1}) - \ell_{t,i_T^*} }.
\end{align*}
The two terms are then typically analyzed separately. \citet{zimmert2020} proved the following bounds for the two terms for Tsallis-INF with reduced-variance loss estimators:
\begin{align*}
        stability &\leq \sum_{t=1}^T \lr{\sum_{i \neq i^*} \frac{\EE[w_{t,i}]^\frac{1}{2}}{2\sqrt{t}} + \frac{\EE[w_{t,i}]}{2\sqrt{t}}} + 14K\log(T) + 15,\\
        penalty &\leq \sum_{t=1}^T \lr{\sum_{i \neq i^*} \frac{\EE[w_{t,i}]^\frac{1}{2}}{2\sqrt{t}} - \frac{\EE[w_{t,i}]}{4\sqrt{t}}} + \frac{3}{4}\sqrt{K}.
\end{align*}
By summation of the two bounds the \regret satisfies
\begin{equation}\label{eq:zimmertlemma}
\overline{Reg}_T \leq \sum_{t=1}^T \lr{\sum_{i \neq i^*} \frac{\EE[w_{t,i}]^\frac{1}{2}}{\sqrt{t}} + \frac{\EE[w_{t,i}]}{4\sqrt{t}}} + 14K\log(T) + \frac{3}{4}\sqrt{K} + 15 .
\end{equation}
Since $\EE[w_{t,i}] \leq \EE[w_{t,i}]^{\frac{1}{2}}$, the \regret of Tsallis-INF with reduced-variance loss estimators satisfies the condition in equation \eqref{eq:sqrt} with $B=\frac{5}{4}$ and $D = \frac{3}{4}\sqrt{K} + 14K\log(T) + 15$. (In the proof of Theorem~\ref{theorem:2} we keep the refined bound on the \regret from equation \eqref{eq:zimmertlemma} to obtain better constants.) Now, after we have shown how the condition in equation \eqref{eq:sqrt} can be satisfied, we present a proof of Theorem~\ref{theorem:3}. We start with a high-level overview of the key ideas and then present the technical details.

\subsection{Overview of the Key Ideas Behind the Proof of Theorem~\ref{theorem:3}}

As observed by \citet{zimmert2020}, for any $\lambda\in[0,1]$ we have
\begin{equation}\label{eq:lambda}
    \overline{Reg}_T = (\lambda+1)\overline{Reg}_T - \lambda \overline{Reg}_T.
\end{equation}\\
The condition on $\Reg$ in equation \eqref{eq:sqrt} can be used to upper bound the first term and the  self-bounding constraint \eqref{eq:self-bounding} to lower bound the second, giving 
\begin{align}
\overline{Reg}_T &\leq (\lambda+1) \left(B \sum_{i\neq i^*}  \sum_{t=1}^{T} \frac{\EE[w_{t,i}]^\frac{1}{2}}{\sqrt{t}} + D\right) - \lambda \left( \sum_{t = 1}^{T} \lr{\sum_{i \neq i^*} \EE[w_{t,i}] \Delta_i}  - C \right)\notag\\
&\leq   \sum_{t=1}^{T}\sum_{i\neq i^*} \left( B(\lambda+1) \frac{\EE[w_{t,i}]^\frac{1}{2}}{\sqrt{t}} - \lambda \EE[w_{t,i}] \Delta_i \right) + \lambda C + (\lambda + 1)D.\label{eq:upperfull}
\end{align}\\
 In the adversarial analysis, we take $\lambda = 0$ and maximize the right hand side of \eqref{eq:upperfull} (which for $\lambda = 0$ is identical to the right hand side of \eqref{eq:sqrt}) under the constraint that $w_{t,i}$ is a probability distribution to obtain $\Ocal(\sqrt{KT})$ regret bound.
This is almost identical to the approach of \citet{zimmert2020}, except that in this case instead of the bound in equation \eqref{eq:sqrt} they use a bound involving summation over all arms, including $i^*$.

 In the self-bounding analysis, \citet{zimmert2020} relax the inequality in \eqref{eq:upperfull} to
 \[
 \Reg \leq \sum_{t=1}^T \sum_{i\neq i^*} \lr{2B \sqrt{\E[w_{t,i}]/t} - \lambda \Delta_i \E[w_{t,i}]} + \lambda C + 2D
 \]
and apply \emph{individual} maximization of each $2B \sqrt{\E[w_{t,i}]/t} - \lambda \Delta_i \E[w_{t,i}]$ term, dropping the constraint that $w_t$ is a probability distribution. We use \eqref{eq:upperfull} directly for bounding the regret and introduce two key novelties: 
\begin{itemize}
    \item[(a)] we keep the constraint that $w_{t}$ are probability distributions; and
    \item[(b)] we jointly optimize with respect to all $w_{t,i}$ and $\lambda$, whereas \citet{zimmert2020} first optimize w.r.t.\ $w_{t,i}$ and then w.r.t.\ $\lambda$.
\end{itemize}
Joint optimization over all $w_{t,i}$ and $\lambda$ under the constraint that $w_t$ are probability distributions is the major technical challenge that we resolve. Our analysis yields three advantages: 
\begin{itemize}
    \item[(A)] The dependence on time is improved from $\log T$ to $\log (T(K-1) / (\sum_{i\neq i^*} \frac{1}{\Delta_i})^2)$ due to (a);
    \item[(B)] We gain the $\sqrt{\log T/\log(T/C)}$ factor due to (b);
    \item[(C)] Our adversarial and stochastic bounds come out of the same optimization problem, highlighting the relation and continuity between the two.
\end{itemize}

\subsection{Proof of Theorem \ref{theorem:3}}
Now we provide a detailed proof of Theorem \ref{theorem:3}. 
\subsubsection*{Proof of the regret bound for an unconstrained adversarial regime (equation \eqref{eq:Badv})} In the unconstrained adversarial regime we take $\lambda = 0$ and plug the inequalities
\begin{equation}\label{eq:quadraticmean}
\sum_{i \neq i^*} \EE[w_{t,i}]^\frac{1}{2} \leq \sqrt{K-1},
\end{equation}
which holds since $\sum_{i \neq i^*} \EE[w_{t,i}] \leq 1$, and $\sum_{t=1}^T \frac{1}{\sqrt{t}} \leq 2\sqrt{T}$ into equation \eqref{eq:upperfull} and obtain the bound in equation \eqref{eq:Badv}. 

\subsubsection*{Proof of the general regret bound for an adversarial regime with a self-bounding constraint (equation \eqref{eq:Bsto})}




In the adversarial regime with a self-bounding constraint, we keep the constraint that $w_t$ is a probability distribution, and thus $\sum_{i \neq i^*} \EE[w_{t,i}]^\frac{1}{2} \leq \sqrt{K-1}$, and apply maximization directly to the sum over $i$ under this constraint.

To simplify the notation, we use $a_{t,i} := \EE[w_{t,i}]^\frac{1}{2}$, $S := \sum_{i \neq i^*}  \frac{1}{\Delta_i}$, and w.l.o.g.\ assume that $i^* = K$. We denote $R_t := \sum_{i\neq i^*} \left( B(\lambda+1) \frac{a_{t,i}}{\sqrt{t}} - \lambda  \Delta_i a_{t,i}^2 \right)$ and $R := \sum_{t=1}^{T} R_t + \lambda C$. With this notation, by  equation \eqref{eq:upperfull} we have
\begin{equation}
\label{eq:RegR}
    \Reg \leq R + (1+\lambda)D.
\end{equation}
We bound $R_t$ under the constraint that $\EE[w_{t,i}]^\frac{1}{2}$ satisfy equation \eqref{eq:quadraticmean}. We have
\begin{align*}
    R_t \leq \max_{a_1,\dots,a_{K-1}} &\sum_{i=1}^{K-1} B(\lambda+1)\frac{a_i}{\sqrt{t}} - \lambda \Delta_i a_i^2\\
    \textrm{s.t.}~~~~ & \sum_{i=1}^{K-1} a_i \leq \sqrt{K-1}.
\end{align*}
By Lemma \ref{lemma:optimization} provided in Appendix \ref{appendix:lemma}, the answer to this optimization problem is as follows:
\begin{enumerate}
    \item If $\frac{B(\lambda+1)S}{2\lambda\sqrt{t}} \leq \sqrt{K-1}$, then $R_t \leq \frac{S B^2(\lambda+1)^2}{4\lambda t}$.
    \item If $\frac{B(\lambda+1)S}{2\lambda\sqrt{t}} \geq \sqrt{K-1}$, then $R_t \leq \frac{\sqrt{K-1}B(\lambda+1)}{\sqrt{t}} - \frac{\lambda(K-1)}{S}$.
\end{enumerate}
This gives a threshold $T_0 = \frac{B^2(\lambda+1)^2S^2}{4\lambda^2 (K-1)}$, so that for $t \leq T_0$ the second case applies to $R_t$, and otherwise the first case applies. We break the time steps into those before $T_0$ and after $T_0$ and obtain:
\begin{align}
R &=  \sum_{t=1}^{T_0} R_t + \sum_{T_0+1}^{T} R_t + \lambda C\notag\\
& \leq \sum_{t=1}^{T_0} \left( \frac{\sqrt{K-1}B(\lambda+1)}{\sqrt{t}} - \frac{\lambda(K-1)}{S} \right) + \sum_{t=T_0+1}^{T} \frac{SB^2(\lambda+1)^2}{4\lambda t} + \lambda C\notag \\
& \leq 2\sqrt{T_0(K-1)} B(\lambda+1) - \frac{\lambda(K-1)T_0}{S} + \frac{SB^2(\lambda+1)^2}{4\lambda} \log \frac{T}{T_0} + \lambda C\notag\\
& = \frac{B^2(\lambda+1)^2S}{\lambda} - \frac{B^2(\lambda+1)^2S}{4\lambda} + \frac{B^2(\lambda+1)^2S}{4\lambda} \left(\log \frac{T(K-1)}{S^2} - 2\log\frac{B(\lambda+1)}{2\lambda}\right) + \lambda C\notag \\
& = \frac{B^2(\lambda+1)^2S}{4\lambda}\left[3 + \log \frac{T(K-1)}{S^2}\right] - \frac{B^2(\lambda+1)^2S}{2\lambda} \log\frac{B(\lambda+1)}{2\lambda} + \lambda C.\label{eq:Rbound}
\end{align}
By taking $\lambda = 1$ we obtain
\[
R \leq B^2 S\left(\log \frac{T(K-1)}{S^2} - 2\log(B) + 3 \right) + C,
\]
which together with \eqref{eq:RegR} gives the bound \eqref{eq:Bsto} in the theorem.

\subsubsection*{Proof of the refined regret bound for an adversarial regime with a self-bounding constraint (equation \eqref{eq:BstoC})}

We continue from equation \eqref{eq:Rbound}. We improve on the bound of \citet{zimmert2020} in equation \eqref{eq:ZS21stoC} by applying a smarter optimization over $\lambda$. We let $\alpha = \frac{2\lambda}{B(\lambda+1)}$ and rewrite the inequality in \eqref{eq:Rbound} as
\begin{equation}\label{eq:alpharegret}
R \leq \underbrace{\frac{B}{2-B\alpha}\underbrace{\left[\frac{S}{\alpha}
\left(3 + \log \left(\frac{T(K-1)}{S^2}\right)\right) + \frac{2S}{\alpha} \log(\alpha) + \alpha C
\right]}_{f(\alpha)}}_{h(B,\alpha)}.
\end{equation}
We denote the right hand side of the expression by $h(B,\alpha)$. We restrict the range of $\alpha$, so that $T \geq T_0 = \frac{S^2}{\alpha^2(K-1)}$, which gives $\alpha \geq \frac{S}{\sqrt{T(K-1)}}$. Since $\lambda\in[0,1]$, we also have $\alpha\leq \frac{1}{B}$. 
In order to bound $h(B,\alpha)$ we need to solve an optimization problem in $\alpha$ over the above interval. However, $h(B,\alpha)$ is not convex in $\alpha$, but we show that the expression in the brackets,  which we denote by $f(\alpha)$, is convex. We take the point $\alpha^* = \argmin_{\alpha \in [\frac{S}{\sqrt{T(K-1)}},\frac{1}{B}]} f(\alpha)$, which achieves the minimum of $f(\alpha)$, and use  $h(B,\alpha^*) = \frac{B}{2-B\alpha^*} f(\alpha^*)$ as an upper bound for $R$. Since $R\leq h(B,\alpha)$ for any $\alpha$, in particular we have $R\leq h(B,\alpha^*)$.

In order to show that $f(\alpha)$ is convex and find its minimum we take the first and second derivatives.
\begin{align*}
f'(\alpha) &= \frac{-1}{\alpha^2} \left[2S\log(\alpha) - C \alpha^2 + S \log\frac{(K-1)T}{S^2} + S \right]  = 0,\\
f''(\alpha) &= \frac{2S}{\alpha^3}\lr{2\log \alpha + \log \frac{T(K-1)}{S^2}}.
\end{align*}
For $\alpha \geq \frac{S}{\sqrt{T(K-1)}}$ the second derivative is positive and, therefore, $f(\alpha)$ is convex and the minimum is achieved when $f'(\alpha) = 0$. This happens when
\[
-\log \frac{\alpha^2(K-1)T}{S^2} + \frac{C}{S}\alpha^2 - 1 = 0.
\]
We define $\beta = \frac{\alpha^2(K-1)T}{S^2}$, then
\[
g(\beta) = \frac{CS}{(K-1)T}\beta - \log(\beta) - 1 = 0.  
\]
Since $\alpha \in [\frac{S}{\sqrt{T(K-1)}},\frac{1}{B}]$, we have $\beta \in [1, \frac{(K-1)T}{B^2S^2}]$. We recall that equation \eqref{eq:BstoC} holds under the assumption that $B^2S\lr{\log \frac{(K-1)T}{B^2S^2} + 1} \leq C \leq \frac{(K-1)T}{S}$. We note that for $C\leq \frac{(K-1)T}{S}$ we have $g(1) = \frac{CS}{(K-1)T} - 1 \leq 0$. We also note that for $C \geq B^2S\lr{\log \frac{(K-1)T}{B^2S^2} + 1}$ we have $g\lr{\frac{(K-1)T}{B^2S^2}} \geq 0$. Since $g(\beta)$ is continuous, the root of $g(\beta)=0$ for $C$ in the above range is thus achieved by $\beta\in[1, \frac{(K-1)T}{B^2S^2}]$ and since $g(\beta)$ is convex the solution is unique.

We find the root of $g(\beta)=0$ by using the  $-1$-branch of the \emph{Lambert W function}, called $W_{-1}(x)$, which is defined as the solution of equation $w e^w = x$. If $g(\beta)=0$, then $\beta$ satisfies
\[
\frac{-CS\beta}{(K-1)T} e^{\frac{-CS\beta}{(K-1)T}} = \frac{-CS}{e(K-1)T},
\]
and thus
\[
\beta = \frac{-T(K-1) }{CS} W_{-1}\lr{\frac{-CS}{e(K-1)T}}.
\]
We conclude that the minimum of $f(\alpha)$ is attained at 
\begin{equation}
\label{eq:alpha-star}
\alpha^* = \sqrt{\frac{-S}{C} W_{-1}\lr{\frac{-CS}{e(K-1)T}}}
\end{equation}
and, consequently, $\log \left(\frac{T(K-1)(\alpha^*)^2}{S^2}\right) = \frac{C}{S}(\alpha^*)^2 - 1$. By substituting this identity into $h(B,\alpha^*)$, we obtain:
\begin{align}
 h(B,\alpha^*) &= \frac{B}{2-B\alpha^*}\left(2\frac{S}{\alpha^*} + 2C\alpha^* \right) \leq B(1+B\alpha^*)\left(\frac{S}{\alpha^*} + C\alpha^* \right)\notag\\ 
&= B\lr{\frac{S}{\alpha^*} + C\alpha^* + B S + BC(\alpha^*)^2}
=  B\lr{\sqrt{\frac{CS}{w}} + \sqrt{C S w} + B S +B S w},\label{eq:lambertsubstitution}
\end{align}
where $w := -W_{-1}\left[\frac{-CS}{e(K-1)T}\right]$ and the inequality follows by the fact that $\forall x \in [0,1]: \frac{2}{2-x} \leq 1+x$. This provides a closed form upper bound for the \regret, but we still need an estimate of $w$ to obtain an explicit bound. We use the result of \citet{lambert}, who provides the following bounds for $W_{-1}(x)$.
\begin{lemma}[\citealt{lambert}]\label{lemma:lambert}
For any $x \leq 1$
\[
1 + \sqrt{2\log(1/x)} + \frac{2}{3}\log(1/x) \leq -W_{-1}(-x/e) \leq 1 + \sqrt{2\log(1/x)} + \log(1/x).
\]
\end{lemma}
To complete the proof it suffices to use Lemma \ref{lemma:lambert} with $x = \frac{CS}{(K-1)T}$, which gives 
\[
1 \leq w \leq 1 + \sqrt{2\log\frac{T(K-1)}{CS}} + \log\frac{T(K-1)}{CS}\leq \lr{1+\sqrt{\log \frac{T(K-1)}{CS}}}^2.
\]
By substituting this into \eqref{eq:lambertsubstitution} we obtain:
\begin{align}
&h(B,\alpha^*) \leq B \sqrt{CS} + B\sqrt{C S}\lr{1 + \sqrt{\log\frac{T(K-1)}{CS}}}\notag \\
&\hspace{2cm}
+ 2B^2S + B^2 S\log\frac{T(K-1)}{CS} + B^2S\sqrt{2\log\frac{T(K-1)}{CS}} \notag\\
&\qquad = B\sqrt{C S}\lr{\sqrt{\log\frac{T(K-1)}{CS}} + 2} + B^2 S\lr{\log\frac{T(K-1)}{CS} + \sqrt{2\log\frac{T(K-1)}{CS}}+ 2}.\label{eq:h-bound}
\end{align}
Finally, by \eqref{eq:alpharegret} we have $R \leq h(B,\alpha^*)$, which together with \eqref{eq:RegR} and the fact that $\lambda \leq 1$ completes the proof.
$\hfill\blacksquare$
\section{Discussion}

We have presented a refined analysis of the Tsallis-INF algorithm in adversarial regimes with a self-bounding constraint. The result improves on prior work in two ways. First, it improves the dependence of the regret bound on time horizon from $\log T$ to $\log \frac{(K-1)T}{(\sum_{i\neq i^*}\frac{1}{\Delta_i})^2}$. Second, it improves the dependence of the regret bound on corruption amount $C$. In particular, for $C=\Theta\lr{\frac{TK}{(\log T)\sum_{i\neq i^*}\sum \frac{1}{\Delta_i}}}$ it improves the \regret bound by a multiplicative factor of $\sqrt{\frac{\log T}{\log\log T}}$. Moreover, we have provided a generalized result that can be used to improve regret bounds for extensions of Tsallis-INF to other problem settings, where the regret satisfies a self-bounding constraint. Due to versatility and rapidly growing popularity of regret analysis based on the self-bounding property, the result provides a powerful tool for tightening regret bounds in a broad range of corrupted settings.

\acks{This project has received funding from European Union's Horizon 2020 research and innovation programme under the Marie Skłodowska-Curie grant agreement No 801199. YS acknowledges partial support by the Independent Research Fund Denmark, grant number 9040-00361B.}

\bibliography{references.bib}
\newpage
\appendix

\section{Technical Lemmas}\label{appendix:lemma}
\begin{lemma}\label{lemma:optimization}
Let $b$ and $c_1,\dots,c_n$ be non-negative real numbers and let
\begin{align*}
Z = \max_{x \in \RR^{n} } & \sum_{i=1}^{n} (bx_{i}  - c_{i} x_i^2 )\\
 s.t. &\sum_{i=1}^{n} x_i \leq M.
\end{align*}
Then
\[
Z = \begin{cases}bM - \frac{M^2}{\sum_{i=1}^{n} \frac{1}{c_i}}, &\text{if~} \sum_{i=1}^{n} \frac{b}{2c_i} >  M,\\\frac{b^2}{4}\sum_{i=1}^{n} \frac{1}{c_i}, &\text{otherwise.}\end{cases}
\]
Moreover, we always have $bM - \frac{M^2}{\sum_{i=1}^{n} \frac{1}{c_i}} \leq \frac{b^2}{4}\sum_{i=1}^{n} \frac{1}{c_i}$ and, therefore, we always have $Z \leq \frac{b^2}{4}\sum_{i=1}^{n} \frac{1}{c_i}$.
\end{lemma}

\begin{proof}
Since $c_i \geq 0$, the objective function is a sum of downward-pointing parabolas and, therefore, concave. Thus, the maximum is attained when the first derivative of the Lagrangian with Lagrange variable $v \geq 0$ for the inequality constraint satisfies
\[
b - 2c_i x_i - v = 0,
\]
where $v (\sum_{i=1}^{n} x_i - M) = 0$. Thus, $x_i =  \frac{b-v}{2c_i}$. The KKT conditions provide two cases:
\begin{itemize}
    \item[i)] If $\sum_{i=1}^{n} \frac{b}{2c_i} >  M$, then $v > 0$  and $\sum_{i=1}^{n} x_i = M$. As a consequence, $v = b - \frac{M}{\sum_{i=1}^{n} \frac{1}{2c_i}}$.
    So $x_i = \frac{M}{c_i \sum_{i=1}^{n} \frac{1}{c_i}}$ and $Z = bM - \frac{M^2}{\sum_{i=1}^{n} \frac{1}{c_i}}$.
\item[ii)] If $\sum_{i=1}^{n} \frac{b}{2c_i} \leq  M$, then $v = 0$ and, as a consequence, $x_i = \frac{b}{2c_i}$ and $Z = \frac{b^2}{4}\sum_{i=1}^{n} \frac{1}{c_i}$.
\end{itemize}
Finally, by the AM-GM inequality we have
\[
\frac{M^2}{\sum_{i=1}^{n} \frac{1}{c_i}} + \frac{b^2}{4}\sum_{i=1}^{n} \frac{1}{c_i} \geq bM,
\]
which gives the final statement of the lemma.
\end{proof}

\noindent
We also use the following result by \citet[Lemma 15]{zimmert2020}.
\begin{lemma}[\citealp{zimmert2020}]
For any $b > 0$ and $c >0$ and $T_0,T \in \mathbb{N}$, such that $T_0 < T$ and $b\sqrt{T_0}>c$, it holds that
\begin{align*}
\sum_{t=T_0+1}^T\frac{1}{bt^\frac{3}{2}-ct} \leq \frac{2}{b\sqrt{T_0}-c}\,.
\end{align*}
\end{lemma}

\noindent By doubling the lower threshold on $b\sqrt{T_0}$ we obtain the following corollary.

\begin{corollary}\label{cor:integral}
For any $b > 0$ and $c >0$ and $T_0,T \in \mathbb{N}$, such that $T_0 < T$ and $b\sqrt{T_0}\geq2c$, it holds that
\[
\sum_{t=T_0+1}^T \frac{1}{b t^\frac{3}{2} - c t} \leq \frac{2}{c}.
\]
\end{corollary}

\section{Proof of Theorem
\ref{theorem:2}}\label{section:appendix}
\begin{proof}
Similar to the proof of Theorem \ref{theorem:3}, for any $\lambda \in [0,1]$ we use the self-bounding constraint and the regret bound of \citet{zimmert2020} given in equation \eqref{eq:zimmertlemma} to provide the following bound for the \regret:
\begin{align*}
        \overline{Reg}_T &= (\lambda+1)\overline{Reg}_T - \lambda \overline{Reg}_T\notag\\
        &\leq (\lambda+1) \left(\sum_{i\neq i^*} \left[ \sum_{t=1}^{T} \frac{\EE[w_{t,i}]}{4\sqrt{t}} + \sum_{t=1}^{T} \frac{\sqrt{\EE[w_{t,i}]}}{\sqrt{t}} \right] + \frac{3}{4}\sqrt{K} + 14K \log(T) + 15 \right)\\
&\quad- \lambda \left( \left[ \sum_{t = 1}^{T} \sum_{i \neq i^*} \EE[w_{t,i}] \Delta_i \right] - C \right).
\end{align*}

As before, to simplify the notation, let $a_{t,i} = \EE[w_{t,i}]^\frac{1}{2}$ and $S = \sum_{i \neq i^*}  \frac{1}{\Delta_i}$ and w.l.o.g.\ assume that $i^* = K$ and define
\begin{align}
R_t =& \sum_{i\neq i^*}\lr{\frac{\lambda+1}{\sqrt{t}} a_{t,i} - \lr{\lambda \Delta_i - \frac{\lambda+1}{4\sqrt{t}}}  a_{t,i}^2}, \label{eq:Rt}\\
R =& \sum_{t=1}^{T} R_t + \lambda C.\notag
\end{align}
Then
\begin{equation}
\label{eq:regret-bound}
\overline{Reg}_T \leq R + (1+\lambda)\lr{\frac{3}{4}\sqrt{K} + 14K \log(T) + 15}. \end{equation}
Hence, in order to obtain a bound for the \regret, it suffices to derive a bound for $R$. We start with the bound for a general adversarial environment and then prove the refinements.

\subsubsection*{Proof of the regret bound for an unconstrained adversarial regime (equation \eqref{eq:adv})}
We take $\lambda = 0$. By plugging it into the definition of $R_t$ in equation \eqref{eq:Rt} we obtain
\[
R_t \leq \frac{\sqrt{K-1}}{\sqrt t} + \frac{1}{4\sqrt{t}}
\]
and
\[
R = \sum_{t=1}^T R_t \leq 2\sqrt{(K-1)T} + \frac{1}{2}\sqrt{T}.
\]
Plugging this into \eqref{eq:regret-bound} completes the proof of \eqref{eq:adv}. 

\subsubsection*{Proof of the regret bounds for an adversarial regime with a self-bounding constraint (equations \eqref{eq:sto} and \eqref{eq:stoC})}

Now we prove the refined bounds for adversarial environments satisfying the self-bounding constraint with unique best arm. 
Similarly to the proof of Theorem \ref{theorem:3}, we bound $R_t$ for each $t \geq 1$ by solving a constrained maximization problem over  $\{a_{t,i}\}_{i=1}^n$, where the constraint is $\sum_{i=1}^{K-1} a_{t,i} \leq \sqrt{K-1}$. But the challenge here is that the coefficients $\lambda \Delta_i - \frac{\lambda+1}{4\sqrt t}$ in front of $a_{t,i}^2$ in the definition of $R_t$ are not necessarily positive, and if they are not, then Lemma~\ref{lemma:optimization} cannot be applied. More precisely, if
\begin{equation}
\label{eq:T1}
\forall i \neq i^*: \lambda\Delta_i \geq \frac{\lambda+1}{4\sqrt{t}} \Rightarrow t \geq \lr{\frac{\lambda+1}{4\lambda \Delta_{min}}}^2,
\end{equation}
where $\Delta_{min} = \min_{i\neq i^*}\{\Delta_i\}$, then all the coefficients are positive. We denote $\alpha = \frac{2\lambda}{\lambda+1}$ and define a threshold $T_1 = \lr{\frac{\lambda+1}{2\lambda \Delta_{min}}}^2 = \lr{\frac{1}{\alpha\Delta_{min}}}^2$. We note that $T_1$ is four times larger than what is required for satisfaction of the condition in equation \eqref{eq:T1}. The reason is that at a later point in the proof we apply Corollary~\ref{cor:integral} for $t\geq T_1$ and we need to satisfy the condition of the corollary. For $t \geq T_1$ we can use Lemma \ref{lemma:optimization} to bound $R_t$. By the lemma we obtain:
\[
R_t \leq \frac{(\lambda+1)^2}{4t}\sum_{i=1}^{K-1}\frac{1}{\lambda \Delta_i - \frac{\lambda+1}{\sqrt t}} = \sum_{i=1}^{K-1}\frac{\lambda+1}{\frac{4\lambda}{\lambda+1}\Delta_i t - \sqrt{t}} = \sum_{i=1}^{K-1}\frac{\lambda+1}{2\alpha\Delta_i t - \sqrt{t}}.
\]
We rewrite each term in the summation in the following way
    \[
    \frac{\lambda+1}{2\alpha\Delta_i t - \sqrt{t}} = \frac{\lambda+1}{2\alpha \Delta_i t} + \frac{\lambda+1}{4\alpha^2\Delta_i^2t^\frac{3}{2} - 2\alpha\Delta_i t}
    \]
and obtain
\begin{equation}
\label{eq:tgeqT1}
\text{for $t\geq T_1$:}\qquad    R_t \leq \frac{S(\lambda+1)}{2\alpha t} + \sum_{i = 1}^{K-1} \frac{\lambda+1}{4\alpha^2\Delta_i^2t^\frac{3}{2} - 2\alpha\Delta_i t}.
\end{equation}

In order to bound $R_t$ for $t < T_1$, we break it into two parts as follows:
\begin{align*}
    R_t =& \sum_{i\neq i^*} \lr{\frac{\lambda+1}{\sqrt{t}} a_{t,i} -  \lambda \Delta_i a_{t,i}^2} + \sum_{i\neq i^*}\lr{\frac{\lambda+1}{4\sqrt{t}} a_{t,i}^2}\\
    \leq& \sum_{i\neq i^*} \lr{\frac{\lambda+1}{\sqrt{t}} a_{t,i} - \lambda \Delta_i a_{t,i}^2} + \frac{1}{2\sqrt{t}},
\end{align*}
where the inequality holds because $\lambda \leq 1$ and $ \sum_{i\neq i^*} a_{t,i}^2 \leq 1$. We use Lemma~\ref{lemma:optimization} to bound the summation in the latter expression. The solution depends on a threshold $T_2 = \frac{(\lambda+1)^2S^2}{4\lambda^2 (K-1)} = \frac{S^2}{(K-1)\alpha^2}$: 
\begin{align}
&\text{for $t\leq T_2$:} \qquad R_t \leq \frac{\sqrt{K-1}(\lambda+1)}{\sqrt{t}} - \frac{\lambda(K-1)}{S} + \frac{1}{2\sqrt{t}},\label{eq:tleqT2}\\
&\text{for $t \geq T_2$:} \qquad R_t \leq \frac{S(\lambda+1)^2}{4\lambda t} + \frac{1}{2\sqrt{t}} = \frac{S(\lambda+1)}{2\alpha t} + \frac{1}{2\sqrt{t}}.\label{eq:tgeqT2}
\end{align}
Note that for $t \geq T_1$ we have a choice between using the bound in equation \eqref{eq:tgeqT1} or one of the bounds in \eqref{eq:tleqT2} or \eqref{eq:tgeqT2}, depending on whether $t\leq T_2$ or $t\geq T_2$. The relation between the thresholds, $T_1 \leq T_2$ or $T_2 \leq T_1$, depends on the relation between $\lr{\frac{1}{\Delta_{min}}}^2$ and $\frac{S^2}{K-1}$. Also note that the choice of $\alpha$ (which determines $\lambda$) affects the thresholds $T_1$ and $T_2$, but not their relation. Similar to the proof of Theorem \ref{theorem:3}, we restrict the range of $\alpha$, so that $T \geq T_2 = \frac{S^2}{\alpha^2(K-1)}$, which gives $\alpha \geq \frac{S}{\sqrt{T(K-1)}}$.

We now derive a bound on $R$. We consider three cases: $T_2  \leq T \leq T_1$, $T_2  \leq T_1 \leq T$, and $T_1 \leq T_2 \leq T$.

\paragraph{First case: $T_2 \leq T \leq T_1$.} By \eqref{eq:tleqT2} and \eqref{eq:tgeqT2} we have:
\begin{align}
&\sum_{t=1}^{T} R_t \leq \sum_{t=1}^{T_2} R_t + \sum_{t=T_2+1}^{T} R_t \notag\\
&\leq \sum_{t=1}^{T_2} \left( \frac{\sqrt{K-1}(\lambda+1)}{\sqrt{t}} - \frac{\lambda(K-1)}{S} \right) + \sum_{t=T_2+1}^{T} \lr{\frac{S(\lambda+1)}{2\alpha t}} + \sqrt{T} \notag\\
&\leq 2\sqrt{T_2(K-1)} (\lambda+1) - \frac{\lambda(K-1)T_2}{S} + \frac{S(\lambda+1)}{2\alpha} \log (\frac{T}{T_2}) + \sqrt{T_1},\label{eq:firstcase}
\end{align}
where in the second line we used $\sum_{t=1}^{T} \frac{1}{2\sqrt{t}} \leq \sqrt{T}$ and in the third line $\sum_{t=T_2+1}^{T} \frac{1}{t} \leq \log(T/T_2)$ and $\lambda \leq 1$ and $T \leq T_1$.

\paragraph{Second case: $T_2 \leq T_1 \leq T$.} By \eqref{eq:tleqT2}, \eqref{eq:tgeqT2}, and \eqref{eq:tgeqT1} we have:
\begin{align}
&\sum_{t=1}^{T} R_t \leq \sum_{t=1}^{T_2} R_t + \sum_{t=T_2+1}^{T_1} R_t + \sum_{t=T_1+1}^{T} R_t\notag\\
&\leq \sum_{t=1}^{T_2} \left( \frac{\sqrt{K-1}(\lambda+1)}{\sqrt{t}} - \frac{\lambda(K-1)}{S} \right) + \sum_{t=T_2+1}^{T} \lr{\frac{S(\lambda+1)}{2\alpha t}} + \sqrt{T_1} + \sum_{i=1}^{K-1} \sum_{t=T_1+1}^T \frac{\lambda+1}{4\alpha^2\Delta_i^2 t^{\frac{3}{2}} - 2\alpha\Delta_i t}\notag\\
&\leq 2\sqrt{T_2(K-1)} (\lambda+1) - \frac{\lambda(K-1)T_2}{S} + \frac{S(\lambda+1)}{2\alpha} \log (\frac{T}{T_2}) + \sqrt{T_1} + \sum_{i=1}^{K-1} \sum_{t=T_1+1}^T \frac{1}{2\alpha^2\Delta_i^2 t^{\frac{3}{2}} - \alpha\Delta_i t},\label{eq:secondcase}
\end{align}
where in the second line we used $\sum_{t=1}^{T_1} \frac{1}{2\sqrt{t}} \leq \sqrt{T_1}$ and in the third line $\sum_{t=T_2+1}^{T} \frac{1}{t} \leq \log(T/T_2)$ and $\lambda \leq 1$.

\paragraph{Third case: $T_1 \leq T_2 \leq T$.} 

By \eqref{eq:tleqT2} and \eqref{eq:tgeqT1} we have:
\begin{align}
&\sum_{t=1}^{T} R_t \leq \sum_{t=1}^{T_2} R_t + \sum_{t=T_2+1}^{T} R_t \notag\\
&\leq \sum_{t=1}^{T_2} \lr{ \frac{\sqrt{K-1}(\lambda+1)}{\sqrt{t}} - \frac{\lambda(K-1)}{S}} + \sqrt{T_2} + \sum_{t=T_2+1}^{T} \lr{\frac{S(\lambda+1)}{2\alpha t}}  + \sum_{i=1}^{K-1} \sum_{t=T_2+1}^T \frac{\lambda+1}{4\alpha^2\Delta_i^2 t^{\frac{3}{2}} - 2\alpha\Delta_i t}\notag\\
&\leq 2\sqrt{T_2(K-1)} (\lambda+1) - \frac{\lambda(K-1)T_2}{S} + \sqrt{T_2} +  \frac{S(\lambda+1)}{2\alpha} \log (\frac{T}{T_2})  + \sum_{i=1}^{K-1} \sum_{t=T_1+1}^T \frac{1}{2\alpha^2\Delta_i^2 t^{\frac{3}{2}} - \alpha\Delta_i t}.\label{eq:thirdcase}
\end{align}

\paragraph{Merging the cases:} Corollary \ref{cor:integral} provides an upper bound for the last terms of \eqref{eq:secondcase} and \eqref{eq:thirdcase}:

\begin{align*}
 \sum_{t=T_1+1}^T \frac{1}{2\alpha^2\Delta_i^2 t^{\frac{3}{2}} - \alpha\Delta_i t} &\leq \frac{2}{\alpha\Delta_i}
.
\end{align*}
Now we combine \eqref{eq:firstcase}, \eqref{eq:secondcase}, and \eqref{eq:thirdcase}, and obtain following bound for $R$:
\begin{align}
R &= \sum_{t=1}^T R_t + \lambda C\notag\\
&\leq 2\sqrt{T_2(K-1)} (\lambda+1) - \frac{\lambda(K-1)T_2}{S} +  \frac{S(\lambda+1)}{2\alpha} \log (\frac{T}{T_2}) + \lambda C\notag\\
&\quad+ \sqrt{\max\{T_1,T_2\}} + \sum_{i=1}^{K-1} \frac{2}{\alpha\Delta_i}.\label{eq:merged}
\end{align}
We note that $\max\lrc{T_1,T_2} = \max\lrc{\frac{S^2}{(K-1)\alpha^2}, \frac{1}{\Delta_{min}^2\alpha^2}}\leq \frac{S^2}{\alpha^2}$. 
Moreover, 
by substituting $T_2 = \frac{S^2}{\alpha^2(K-1)}$ into \eqref{eq:merged}  we obtain:
\begin{align}
R &\leq 2(\lambda+1)\frac{S}{\alpha} - \frac{\lambda S}{\alpha^2} + \frac{S(\lambda+1)}{2\alpha} \log\lr {\frac{\alpha^2(K-1)T}{S^2}} + \lambda C + \frac{S}{\alpha} + \sum_{i=1}^{K-1} \frac{2}{\alpha\Delta_i}
\notag\\&= \frac{\lambda+1}{2}\lrs{
4\frac{S}{\alpha} - \frac{S}{\alpha} + \frac{S}{\alpha} \log\lr {\frac{(K-1)T}{S^2}} + \frac{2S}{\alpha} \log(\alpha) + \alpha C}
 + \frac{3S}{\alpha}
 \notag\\ 
&= \underbrace{\frac{1}{2-\alpha}\left[\frac{S}{\alpha}
\left(3 + \log \left(\frac{T(K-1)}{S^2}\right)\right) + \frac{2S}{\alpha} \log(\alpha) + \alpha C
\right]}_{h(1,\alpha)} + \frac{3S}{\alpha}
.\label{eq:extraterms}
\end{align}
We recognize that the first term in equation \eqref{eq:extraterms} is $h(1,\alpha)$, which was defined earlier in equation \eqref{eq:alpharegret}. 


\paragraph{Proof of the general bound in equation \eqref{eq:sto}:} By taking $\lambda = 1$, which corresponds to $\alpha = 1$, we obtain
\begin{align*}
R &\leq S\left(\log \left(\frac{T(K-1)}{S^2}\right) + 3 \right) + C +3S\\
&= S\left(\log \left(\frac{T(K-1)}{S^2}\right) + 6 \right) + C.
\end{align*}
Plugging this and the value of $\lambda$ into \eqref{eq:RegR} completes the proof of \eqref{eq:sto}.

\paragraph{Proof of the refined bound in equation \eqref{eq:stoC}:}
We note that the range of $C$ in the refined bound in equation \eqref{eq:stoC} is the same as in the refined bound in \eqref{eq:BstoC} in Theorem~\ref{theorem:3} for $B=1$. We take $\alpha^*$ as in equation \eqref{eq:alpha-star}, i.e., $\alpha^* = \sqrt{\frac{-S}{C} W_{-1}\lr{\frac{-CS}{e(K-1)T})}}$. By Lemma~\ref{lemma:lambert} we have $-W_{-1}\lr{\frac{-CS}{e(K-1)T})} \geq 1$, and thus $\alpha^* \geq \sqrt{\frac{S}{C}}$. By plugging this bound and the bound on $h(1,\alpha^*)$ from  equation \eqref{eq:h-bound} into equation \eqref{eq:extraterms}, we obtain:
\begin{align*}
    R &\leq \sqrt{C S}\lr{\sqrt{\log\frac{T(K-1)}{CS}} + 2} + S\lr{\log\frac{T(K-1)}{CS} + \sqrt{2\log\frac{T(K-1)}{CS}}+ 2} + 3\sqrt{CS}\\
    &= \sqrt{C S}\lr{\sqrt{\log\frac{T(K-1)}{CS}} + 5} + S\lr{\log\frac{T(K-1)}{CS} + \sqrt{2\log\frac{T(K-1)}{CS}}+ 2}.
\end{align*}
Plugging this bound into \eqref{eq:RegR} and using the fact that $\lambda \leq 1$ completes the proof of \eqref{eq:stoC}.
\end{proof}

\end{document}